\documentclass[sigconf]{aamas} 

\usepackage{balance} 
\usepackage{algorithm}
\usepackage{algorithmic}
\usepackage{amsmath}
\usepackage{amsthm}
\newtheorem{definition}{Definition}
\usepackage{latexsym}

\usepackage{amssymb}
\usepackage{booktabs}
\usepackage{enumitem}
\usepackage{graphicx}
\usepackage{color}
\usepackage{multirow}
\usepackage[table,xcdraw]{xcolor}
\usepackage{threeparttable}
\usepackage{adjustbox}
\usepackage{float}
\usepackage{appendix}

\setcopyright{ifaamas}
\acmConference[AAMAS '26]{Proc.\@ of the 25th International Conference
on Autonomous Agents and Multiagent Systems (AAMAS 2026)}{May 25 -- 29, 2026}
{Paphos, Cyprus}{C.~Amato, L.~Dennis, V.~Mascardi, J.~Thangarajah (eds.)}
\copyrightyear{2026}
\acmYear{2026}
\acmDOI{}
\acmPrice{}
\acmISBN{}

\acmSubmissionID{<<1178>>}

\title{TCRL: Temporal-Coupled Adversarial Training for Robust Constrained Reinforcement Learning in Worst-Case Scenarios}

\author{Wentao Xu}
\affiliation{
  \institution{Northeastern University}
  \city{Shenyang}
  \country{China}}
\email{20223441@stu.neu.edu.cn}
\author{Zhongming Yao}
\authornote{Corresponding author}
\affiliation{
  \institution{Northeastern University}
  \city{Shenyang}
  \country{China}}
\email{yaozming@stumail.neu.edu.cn}
\author{Weihao Li}
\affiliation{
  \institution{Northeastern University}
  \city{Shenyang}
  \country{China}}
\email{20223284@stu.neu.edu.cn}
\author{Zhenghang Song}
\affiliation{
  \institution{Zhejiang University}
  \city{Ningbo}
  \country{China}}
\email{22451060@zju.edu.cn}

\author{Yumeng Song}
\affiliation{
  \institution{Aalborg University}
  \city{Aalborg}
  \country{Denmark}}
\email{yumengs@cs.aau.dk}
\author{Tianyi Li}
\affiliation{
  \institution{Aalborg University}
  \city{Aalborg}
  \country{Denmark}}
\email{tianyi@cs.aau.dk}
\author{Yushuai Li}
\affiliation{
  \institution{Aalborg University}
  \city{Aalborg}
  \country{Denmark}}
\email{yusli@cs.aau.dk}

\begin{abstract}
Constrained Reinforcement Learning (CRL) aims to optimize decision-making policies under constraint conditions, making it highly applicable to safety-critical domains such as autonomous driving, robotics, and power grid management. However, existing robust CRL approaches predominantly focus on single-step perturbations and temporal-independent adversarial models, lacking explicit modeling of temporal-coupled perturbations robustness.To tackle these challenges, we propose \texttt{TCRL}, a novel \underline{\textbf{t}}emporal-coupled adversarial training framework for robust \underline{\textbf{c}}onstrained \underline{\textbf{r}}einforcement \underline{\textbf{l}}earning (\texttt{TCRL}) in worst-case scenarios. First, \texttt{TCRL} introduces a worst-case-perceived cost constraint function that estimates safety costs under temporal-coupled perturbations without the need to explicitly model adversarial attackers.
Second, \texttt{TCRL} establishes a dual-constraint defense mechanism towards the reward to counter temporal-coupled adversaries while maintaining the unpredictability of the reward.
The experimental results demonstrate that \texttt{TCRL} consistently outperforms existing methods in terms of robustness against temporal-coupled perturbation attacks across a variety of CRL tasks.
\end{abstract}


\newcommand{\BibTeX}{\rm B\kern-.05em{\sc i\kern-.025em b}\kern-.08em\TeX}

\begin{document}


\pagestyle{fancy}
\fancyhead{}


\maketitle 


\section{Introduction}

Constrained reinforcement learning (CRL) is an important branch of traditional reinforcement learning (RL). In recent years, CRL has attracted widespread attention due to its unique ability to optimize decision-making policies under constraint conditions, maximizing task rewards while ensuring that system constraints are not violated. CRL has been widely applied in high-risk domains such as autonomous driving~\cite{zhang2024safe}, robotic control~\cite{roza2022safe}, and power grid management~\cite{mo2021safe,bui2025critical}.
However, when CRL is deployed in real-world physical environments, agents often encounter external adversarial perturbations, which expose significant shortcomings in the robustness of existing methods~\cite{li2023robust}. Existing work~\cite{liu2022robustness,meng2023integrating} has demonstrated that traditional CRL exhibits pronounced vulnerability to adversarial attacks, thereby underscoring the pressing need to enhance robustness of CRL for reliable deployment in real-world applications.

While robust RL has introduced some approaches to address environmental uncertainties~\cite{brunke2022safe,garcia2015comprehensive,he2023robotic,yang2022cup,liu2022constrained}, these approaches are not directly applicable to CRL.
The reason is the inherent conflict between safety constraints and reward maximization. In CRL, satisfying safety constraints is typically prioritized over optimizing rewards. In contrast, agents in robust RL tend to aggressively pursue reward maximization, often violating critical safety constraints during exploration, especially under adversarial conditions.
Existing studies~\cite{zhang2020robust,mankowitz2020robust} in CRL largely address robustness by modeling and defending against adversarial input perturbations under constraint conditions.
However, they are limited to addressing temporal-independent observational perturbations (e.g., bounded within a fixed-radius $\ell_{p}$ norm ball) and fail to consider robustness under worst-case temporal-coupled attack scenarios. 
In real-world environments, the state observations are often be corrupted by temporal-coupled perturbations. Attackers can generate strong temporal perturbations by learning the safety cost constraint function and reward function~\cite{jin2013reducing}. The attack intensity progressively increases over time, causing the agent to deviate from its target and potentially exhibit dangerous behaviors due to the influence of temporal-coupled perturbations\footnotemark, which can lead to significant reward degradation and potentially dangerous outcomes. 
Therefore, existing robust CRL methods fail to perceive and defend against optimal temporal-coupled adversaries. They face two challenges:

\footnotetext{Appendix provides two examples to illustrate the importance of considering temporal-coupled state perturbations.}

\noindent 1) \textbf{\textit{How to enable cost constraint functions to perceive worst-case temporal-coupled adversaries?}} Existing studies~\cite{lyu2019advances,zhang2021robust} model the worst-case attacker against the cost constraint function in CRL as a RL problem, wherein training agents against the learned attacker yields robust policies.
However, in real-world environments, adversaries do not co-train with the agent but rather directly corrupt the state observations. Consequently, CRL trained against a fixed temporal-coupled attack strategy exhibit limited robustness when facing uncertain adversaries, which may result in safety constraint violations.

\noindent 2) \textbf{\textit{How to effectively conduct robust training of CRL agents under temporal-coupled perturbation attacks?}}
Existing robust CRL methods~\cite{morimoto2005robust,tessler2019action} have been predominantly designed to address single-step perturbations, without explicitly modeling temporal robustness. Furthermore, the development of observational adversarial attack in existing methods fails to account for temporal-coupled effects~\cite{liu2022robustness,zhang2021robust}. These limitations prevent effective quantification of robustness under temporal-coupled perturbations and cause existing CRL defense mechanisms to collapse when facing such attacks.

To address these two challenges, we propose \texttt{TCRL}, a temporal-coupled adversarial training framework for robust constrained reinforcement learning in worst-case scenarios, which enables the agent to maintain high rewards while effectively avoiding violations of safety constraints under temporal-coupled attacks.
In response to the first challenge, this component incorporates a perceptual mechanism that explicitly accounts for the temporal-coupled nature of potential perturbations, allowing the agent to optimize its policy with respect to the worst-case safety costs within the temporal perturbation boundaries. By doing so, the safety constraints are effectively enforced without requiring explicit modeling of adversarial policy, thereby improving robustness against temporally-structured attacks on constraint signals.
To tackle the second challenge,
we establish a defense mechanism against temporal-coupled attacks towards the reward during the robust CRL training process. Specifically, based on the accumulated rewards over a time window, we introduce two reward constraints to disrupt the attacker's temporal-coupled patterns while maintaining the unpredictability of the reward. This further enhances the agent's robustness against temporal-coupled perturbation attacks.
The main contributions of this work are summarized as follows:
\begin{itemize}[leftmargin=*]
\item We propose a temporal-coupled adversarial training framework for robust constrained reinforcement learning in worst-case scenarios, \texttt{TCRL}, to defend against temporal-coupled perturbation attacks, enabling the agent to achieve high rewards without violating safety constraints.
\item We design a cost constraint function capable of directly estimating the worst-case safety constraints cost under temporal-coupled perturbations; based on this, we construct a loss function tailored for worst-case safety constraints cost robust training.
\item We introduce a dual-constraint defense mechanism towards the reward: a temporal correlation constraint and a reward entropy stability constraint.  The temporal correlation constraint disrupts the attacker's temporal-coupled attack patterns, while the entropy constraint preserves the unpredictability of the reward.
\item We conduct extensive experiments across multiple CRL tasks. Extensive experiments in four tasks have demonstrated that \texttt{TCRL} reduces the safety cost under temporal-coupled perturbation attacks by up to 190.77 times and increases the reward value up to 1.34 times. The result demonstrates that \texttt{TCRL} achieves superior robustness against temporal-coupled perturbation attacks.  
\end{itemize}


\begin{figure*}[ht]
\centering
\includegraphics[width=1.00\textwidth]{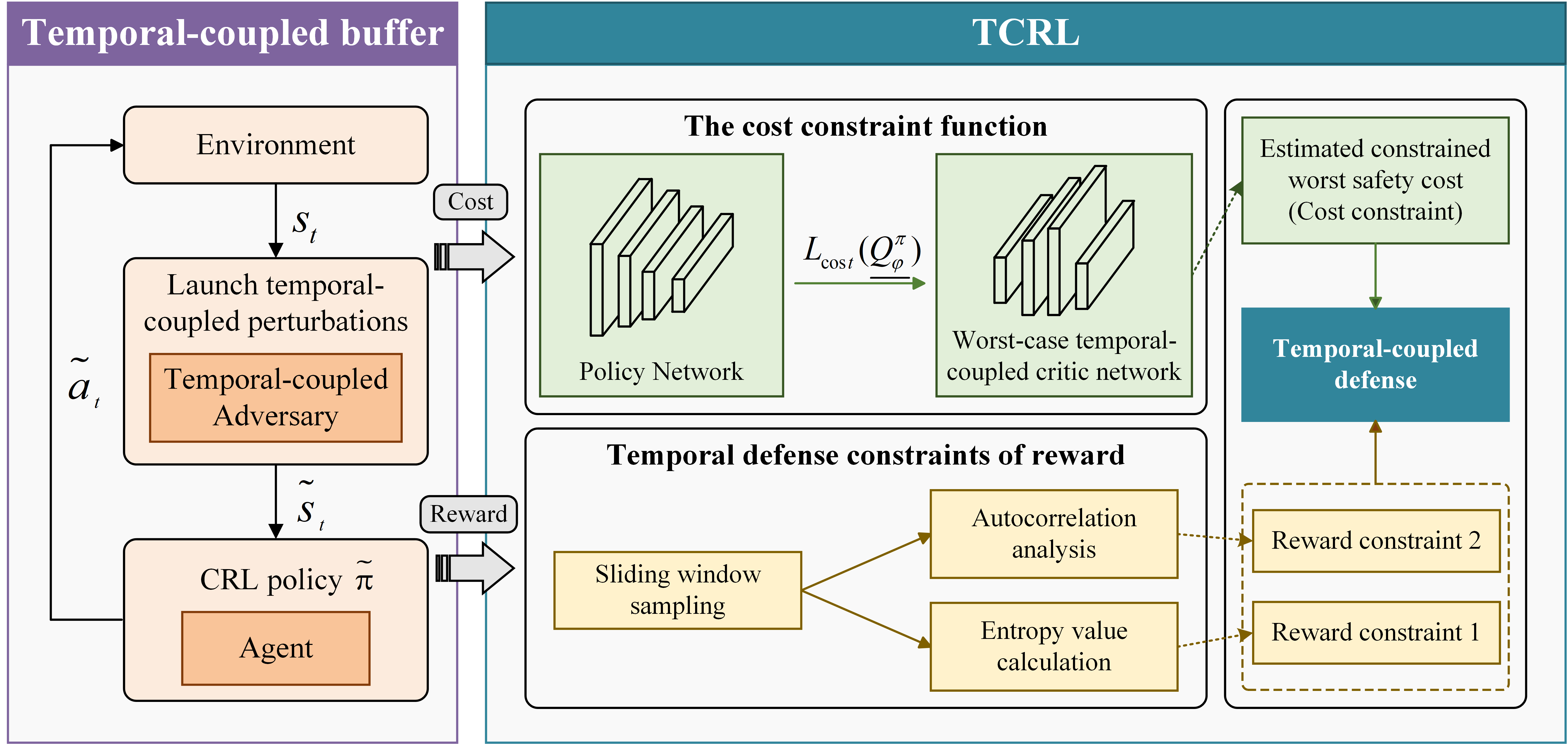}
\caption{\texttt{TCRL}'s framework.}

\label{fig:eurai}
\end{figure*}


\section{Related work}

\noindent\textbf{Robust CRL.}
CRL is commonly modeled using the constrained Markov decision process (CMDP) framework, where the goal is to learn a policy that maximizes expected rewards while satisfying predefined cost constraints~\cite{mankowitz2020robust,varagapriya2023joint,varagapriya2022constrained,borkar2014risk}.
A general approach to solving CMDPs is the Lagrangian-based method, which introduces additional Lagrange multipliers to penalize constraint violations~\cite{guo2011total,chow2018risk}. Another representative method is the Constrained Policy Optimization (CPO) algorithm~\cite{achiam2017constrained}, which ensures that the agent satisfies the constraints at every step during the learning process.
While these approaches focus on ensuring constraint satisfaction during training, they do not explicitly address the robustness of CRL under adversarial or uncertain perturbations.
Early studies on CRL robustness primarily addressed deterministic perturbations~\cite{ng2003integration,asadpoure2011robust}, whereas recent mainstream approaches focus on stochastic perturbations constrained within bounded regions~\cite{shi2024achieving,xie2005stochastic,feng2020robust}.

\noindent\textbf{Defending against Adversarial Perturbations on State Observations.}
 Some studies~\cite{zhang2020robust,shen2020deep,oikarinen2021robust,lutjens2020certified,oikarinen2021robust,fischer2019online} have investigated defense mechanisms against adversarial perturbations to state observations. Regularization-based methods ~\cite{zhang2020robust,shen2020deep,oikarinen2021robust}  enforce policies with similar outputs for similar inputs and have demonstrated verifiable improvements in the performance of deep $Q$ -networks (DQNs) on several Atari games. These methods may not reliably improve worst-case performance in continuous control tasks. 
Some studies aims to enhance the robustness of reinforcement learning by analyzing the boundary behavior of the action-value function $Q$. Specifically, existing studies~\cite{lutjens2020certified,oikarinen2021robust}  guarantee the robustness of single-step action selection by computing the lower bound of the $Q$ function. 
These methods fail to model adversarial attack-induced distributional bias in state transitions, causing even "safe" current actions to lead to vulnerable future states. Additionally, verification frameworks for discrete action spaces are hard to generalize to continuous control settings, with their inherent combinatorial explosion severely limiting practical applicability.
An existing study~\cite{fischer2019online} introduces a principled framework for robust reinforcement learning that formalizes worst-case robustness and can be incorporated into existing dynamic learning algorithms to improve their resilience.

Existing robust CRL methods focus on defending the reward function against adversarial attacks. However, these methods often fail to account for the stability of constraint satisfaction under adversarial perturbations, especially in complex settings involving continuous state-action spaces and temporally correlated attacks. 
\texttt{TCRL} effectively estimates the worst-case safety cost under temporally coupled perturbations by constructing a cost constraint function capable of perceiving worst-case scenarios. Meanwhile, \texttt{TCRL} employs a reward-based temporal defense constraint to mitigate such attacks, thereby enhancing the overall robustness of constrained reinforcement learning.


\section{Preliminary}

\subsection{CRL.} 
In the CRL, the process of the agent interacting with the environment is formulated as a Constrained Markov Decision Process (CMDP)~\cite{altman1998constrained}. The CMDP is defined as the tuple $\langle\mathcal{S}, \mathcal{A}, P, r, c, \gamma, \xi\rangle$, where $\mathcal{S}$ is the state space, $\mathcal{A}$ is the action space, $P(s'\mid s, a)$ is the important state transfer probability, $r \colon S \times A \to \mathbb{R}$ is the reward function, $c \colon S \times A \to \mathbb{R}^+$ is the immediate cost function, $\gamma \in [0,1)$ is the discount factor, and $\xi \colon \mathcal{S} \to \mathbb{R}^+$ is the initial state distribution. We denote this type of CRL problem as $\mathcal{N}_{\Pi}^{\eta}$, where $\eta$ is the preset safety constraint threshold. $ \Pi: \mathcal{S} \times \mathcal{A} \rightarrow[0,1]$ is the strategy class. The reward value function is $V_{r}^{\pi}(s_{t})=\mathbb{E}_{\pi,P} \left[ \sum_{t=0}^{\infty} \gamma^t r(s_t, a_t) |s_0=s\right]$ and the cost value function is $V_{c}^{\pi}(s_{t})=\mathbb{E}_{\pi,P} \left[ \sum_{t=0}^{\infty} \gamma^t c(s_t, a_t)|s_0=s \right]$, and the corresponding action-value function containing both reward and cost is $Q_{r}^{\pi}(s_t,a_t)=\mathbb{E}_{\pi,P} \left[ \sum_{t=0}^{\infty} \gamma^t r(s_t, a_t)  |s_0=s,a_0=a\right]$, $Q_{c}^{\pi}(s_t,a_t)=\mathbb{E}_{\pi,P} \left[ \sum_{t=0}^{\infty} \gamma^ t c(s_t, a_t) |s_0=s,a_0=a\right]$. The target strategy $\pi ^{\ast } $ of $\mathcal{N}_{\Pi}^{\eta}$ needs to be satisfied:
\begin{eqnarray}\label{eq:optimal_policy}
\pi^* & = & \arg\max_{\pi \in \Pi} V_{r}^{\pi}(s_{t})\ ,\text{s.t.}  V_{c}^{\pi}(s_{t}) \leq \eta.
\end{eqnarray} 
\subsection{Modeling State Observation Perturbation Attacks.}

In the CRL, the test-phase state observation perturbation attack serves as a real-time disrupter to manipulate the state observation channel of the intelligences through a bounded perturbation of the $\ell_{p}$ paradigm: at each time step $t$, the attacker perturbs the true state returned by the environment, $s_t$, into $\tilde{s}_{t} \in B_{\ell_{p}}\left(s_{t}, \epsilon\right)$, where $B_{\ell_p}(s_t, \epsilon) = \left\{ \tilde{s} \in \mathcal{S} \mid \|\tilde{s} - s_t\|_p \leq \epsilon \right\}$ is a perturbed neighborhood of p-norm constraints, and $\tilde{\pi}$ is the perturbed strategy.


\section{Methodology}
  
\subsection{Framework}
To solve the robustness limitations of constrained CRL training methods under temporal-coupled  perturbation attacks, we propose TCRL, a novel adversarial training framework tailored for robust constrained reinforcement learning in worst-case scenarios. Fig.~\ref{fig:eurai} illustrates the framework of \texttt{TCRL}. 
When the agent is subjected to a temporal-coupled perturbation attack, we collect the safety cost and reward of the agent.
In order to reduce the cost and enhance the reward, we perform two steps: (ii) constructing a cost constraint function and (ii) implementing a dual-constraint defense mechanism.
In the first step, we introduce a network designed to estimate the worst-case safety cost, which enables the construction of a specialized safety constraint function.
This cost constraint function enables effective estimation of the worst-case safety cost without explicitly constructing an attacker.
In the second step, we establish a dual-constraint defense mechanism on reward signals to effectively counteract temporal-coupled perturbation attacks. This mechanism ensures that the policy attains high rewards while satisfying safety cost constraints, thereby significantly enhancing the robustness of CRL.

\subsection{Temporal-coupled Perturbation Attack}

To introduce temporal coupling into the adversarial perturbations, we propose a novel perturbation method that incorporates perturbation intensity along the time dimension. 

We first design a temporal-coupled perturbation constraint parameter as a dynamic parameter so that it is automatically adjusted according to the magnitude of the perturbation at the previous moment.

\begin{definition}[Temporal-coupled Perturbation Constraint Parameter]\label{thm:perturbation}
 For time step $t$, the adaptive  temporal-coupled constraint parameter $\overline{\epsilon}$ is defined as
\begin{eqnarray}\label{eq:Bellman}
\bar{\epsilon}_{t}=\bar{\epsilon}_{t-1}+\alpha \cdot\left\|p_{t-1}\right\|, 
\end{eqnarray}
where $\alpha$ is the perturbation adjustment factor, $p_{t-1}$ is the perturbation vector at the previous moment, and $\left \| p_{t-1} \right \| $ is the magnitude of the perturbation at the previous moment $\ell _p$. 
\end{definition}

Based on the temporal-coupled perturbation constraint parameter, a temporal-coupled perturbation constraint is formulated as follows:
\begin{eqnarray}\label{eq:time}
\left\|p_{t}-p_{t-1}\right\| \leq \bar{\epsilon}_{t}.
\end{eqnarray}

We introduce the parameter $\alpha \in (0,1)$ to modulate sensitivity to the magnitude of perturbations. Here, $\bar{\epsilon}_{t}$ represents the update constraint on the perturbation magnitude. Since $\alpha >0$, $\bar{\epsilon}_{t}$ increases over a time period, enabling attackers to explore higher-intensity perturbation strategies.

Based on the constraint parameter $\bar{\epsilon}_t$, we further propose a temporal-coupled state observation perturbation method.

\begin{definition}[$\bar{\epsilon}_{t}$-Temporal-coupled State Perturbation]\label{thm:state perturbation}
In the context of temporal-coupled state perturbations, a perturbation $p_t$ is considered valid if it adheres to the temporal-coupled constraint $\bar{\epsilon}_{t}$, which is defined as $\bar{\epsilon}_{t}$:$\left \| s_{t}-\tilde{s}_{t}-\left(s_{t+1}-\tilde{s}_{t+1}\right) \right \| \le \bar{\epsilon}_{t}$, where $s_t$ and $s_{t+1}$ are the perturbed states obtained by adding $p_t$ and $p_{t+1}$ to $s_t$ and $s_{t+1}$ respectively.
\end{definition}

This perturbation method allows the adversary to incorporate temporal coupling across time steps.
By constraining perturbations in a bounded range and precluding abrupt alterations in direction, the attacker can execute successive and more potent attacks while maintaining a measure of stability.

\subsection{Cost Constraint Functions}

The goal of CRL is to find an optimal policy that maximizes the state value function $V$ subject to safety constraints.
However, the presence of observation noise and adversarial perturbations in real-world deployment environments poses a significant challenge to safety.
When state observations are compromised by sophisticated adversarial perturbations (e.g., MC and MR attackers proposed by Liu et al.~\cite{liu2022robustness}), relying solely on the nominal state-value function $V$ and action-value function $Q$ in conventional evaluation paradigms becomes insufficient to ensure safety.
To address this limitation, we propose a safety constraint evaluation method for robust CRL. 
In this method, when the observed spatial perturbations remain within a predefined attack budget, the worst-case safety cost trajectory generated by the policy over the entire decision horizon must consistently satisfy a predefined safety threshold.
This method provides a quantifiable criterion for evaluating the robustness of a policy in adversarial scenarios.

First, we formulate the single-step worst-case action cost value across different environments as follows:
\begin{eqnarray}\label{eq:Worst_c}
\underline{Q}^{\pi}_c\left(s,a\right) &:=& \mathbb{E}_{\pi,P} \left[ \sum_{t=0}^{\infty} \gamma^{t} c(s_{t},a^{\ast}_{t})  \right],
\end{eqnarray}
where $a^*_t$ is the worst-case action taken by state $s_t$ under a temporal-coupled perturbation attack and $\pi$ is the policy of agents. The worst-case cost value $\underline{V}^{\pi}_c$ can be define in the same way

To estimate the worst-case costs, we introduce a new worst-case cost Bellman operator.

\begin{definition}[Worst-case Cost Bellman Operator]\label{thm:Bellman}
In a given MDP, for a deterministic policy $\pi$ and a deterministic perturbation, the worst-case cost Bellman operator is defined as follows:
\begin{eqnarray}\label{eq:Bellman}
\underline{ \Gamma}^{\pi}Q_c(s, a) = \mathbb{E}_{s^{\prime} \sim P(s'\mid s, a)}\nonumber\\ \Bigg[ c(s, a) + \gamma \max_{a^{\prime} \in {\Omega}(s^{\prime}, \pi)} \underline{Q}^{\pi}_c(s^{\prime}, a^{\prime}) \Bigg],
\end{eqnarray}
where $\forall s\in\mathcal{S}$, $\Omega(s, \pi)$ is the set of all actions that the attacker may output by perturbing the inputs to $\tilde{s}\in\mathcal{B}_{\epsilon}(s')$ by perturbing the strategy $\pi$ in state $s'$.

\end{definition}

We formulate the action set outputted by corrupted states as follows:
\begin{eqnarray}\label{eq:adv}
{\Omega}(s, \pi):=\left\{a \in \mathcal{A}: \exists \tilde{s} \in \mathcal{B}_{\epsilon}(s)\right. s.t. \left.\pi(\tilde{s})=a\right\}.
\end{eqnarray}
The process of computing ${{\Omega}}_{\text{adv}}(s, \pi)$ involves identifying the potential actions that the policy $\pi$ might produce when the state $s$ undergoes temporal-coupled perturbations confined within $\mathcal{B}_{\epsilon}(s)$. This identification is achieved by leveraging convex relaxation techniques, which have been extensively employed in the analysis of neural networks  \cite{9010971,gowal2019effectivenessintervalboundpropagation,liang2022efficientadversarialtrainingattacking}. The details are in the Appendix.

We introduce the following theorem that relates the worst-case cost Bellman operator to the worst-case cost value function.

\begin{theorem}[Worst-case Cost Bellman Operator with Worst-case Cost Values]\label{thm:cost}
For any strategy $\pi$, the operator $\Gamma$ is a compression mapping whose unique immovable point ${Q}^{\pi}_c$ is the worst-cost value of the strategy under an attack on the bounded observation of the $\ell _p$-paradigm, where the upper bound on the attack radius is $\epsilon$.
\end{theorem}
\begin{proof}
A full proof can be found in the Appendix.
\end{proof}

The proof of theorem~\ref{thm:cost} in the Appendix shows that the highest possible cumulative cost of a strategy under a bounded observation attack can be computed by the worst-case cost Bellman operator. To compute the worst-case cost value, we train a worst-case cost network $\underline{Q}^{\pi}_{\varphi}$. Given a mini-batch $\left\{s_{t}, a_{t}, r_{t}, s_{t+1}\right\}_{t=1}^{N}$, $\underline{Q}^{\pi}_{\varphi}$ is optimized using the following loss function:
\begin{eqnarray}\label{eq:cost loss training}
\mathcal{L}_{\text {cost }}\left(\underline{Q}^{\pi}_{\varphi}\right):=\frac{1}{N} \sum_{t=1}^{N}\left(\underline{c}_{t}-\underline{Q}^{\pi}_{\varphi}\left(s_{t}, a_{t}\right)\right)^{2}, 
\end{eqnarray}
where $\underline{c}_{t}=c_{t}+\gamma \min _{a^{\prime} \in {\Omega}\left(s_{t+1}, \pi\right)} \underline{Q}^{\pi}_{\varphi}\left(s_{t+1}, {a^{\prime}}\right)$.

The corresponding worst-cost value is computed as $\underline{V}^{\pi}_c(s)=\max _{a \in {\Omega}(s, \pi)} \underline{Q}^{\pi}_{\varphi}(s, a)$.
Consequently, the original CRL constraint (expected cost $\le$$ C_{threshold}$) is transformed into a worst-case constraint:
\begin{eqnarray}\label{eq:yueshu}
\quad \underline{V}^{\pi}_c \leq C_{\text {threshold}}.
\end{eqnarray}

\subsection{Reward-based Temporal Defense Constraints}

Temporal-coupled attackers targeting CRL policies exploit an agent's historical reward data to model its reward function, generating perturbations that reduce rewards and intensify over time. 
To ensure CRL policy optimality under attack, a dual temporal attacker defense mechanism with two reward-based constraints is proposed: 
(i)limiting reward temporal correlation to hinder attackers from accurately modeling the function via historical data and preventing continuous reward degradation; 
and (ii) maintaining reward function unpredictability, ensuring the agent's original policy does not change significantly even if attackers manipulate rewards at a time step, thus stopping attackers from predicting subsequent rewards.

\noindent\textbf{Autocorrelation Constraints Based on Reward Values.}\label{sec:constrains}
To address temporal-coupled perturbation attacks from adversaries in adversarial environments, we propose a reward-based temporal decoupling optimization method with an autocorrelation constraint.
This method enhances strategy robustness by imposing a weak correlation constraint on the reward sequence, thereby disrupting the adversary's ability to model temporal coupling in system dynamics.

Let the agent's strategy be $\pi$,
the corrupted strategy is $\tilde{\pi}$, the attacker causes the agent to perform the action $\tilde{a}_t\in{\Omega}(\tilde{s}_t, \tilde{\pi})$ by perturbing the observed state $\tilde{s}_t\in\mathcal{B}_{\epsilon}(s)$, and the obtained reward value of $\tilde{r}_t$.
The temporal correlation penalty term in the time window $w$ is:
\begin{eqnarray}\label{eq:reward constraint1}
C_{\mathrm{corr}}=\frac{1}{w^{2}} \sum_{k=1}^{w} \sum_{l=1}^{k}\left|\phi\left(\tilde{r}_{t}, \tilde{r}_{t-l}\right)\right|\leq\epsilon_{corr},
\end{eqnarray}
where $\phi$ is the autocorrelation function.

The autocorrelation function $\phi:\mathbb{R}^{2}\rightarrow[-1,1]$ is in standardized covariance form:
\begin{eqnarray}\label{eq:zixiangguan}
\phi\left(\tilde{r}_{t}, \tilde{r}_{t-l}\right)=\frac{\mathbb{E}\left[\left(\tilde{r}_{t}-\mu_{t}\right)\left(\tilde{r}_{t-l}-\mu_{t-l}\right)\right]}{\sigma_{t} \sigma_{t-l}},
\end{eqnarray}

where $\mu_{t}=\mathbb{E}[\tilde{r}_t]$ and $\sigma^{2}_{t}=Var(\tilde{r}_t)$ denote the expectation and and variance of the reward received by the agent under temporally coupled perturbation attacks, respectively.
By enforcing the autocorrelation coefficient to remain below a threshold $\epsilon_{corr}$, this constraint restricts the attacker's ability to exploit temporal dependencies, thereby disrupting temporal-coupled perturbation mechanisms.

To overcome the optimization challenges arising from the non-differentiability of the conventional correlation coefficient, a microscopic approximation with a gradient correction mechanism is proposed:
\begin{eqnarray}\label{eq:gradient}
\nabla_{\theta} \mathcal{L}_{\text {corr }}=\lambda \sum_{l=1}^{L} \frac{\partial}{\partial \theta}\left[\phi\left(\tilde{r}_{t}, \tilde{r}_{t-l}\right)^{2}\right],
\end{eqnarray}
where $\lambda$ is the regularization factor.

\noindent\textbf{Entropy Stability Constraints Based on Reward Values.}
To effectively regulate the exploration behavior of strategies in dynamic environments and to prevent instability caused by abrupt changes in the reward distribution, we introduce an entropy stability constraint. 
In practice, this constraint mitigates over-adjustment to environmental changes, thereby preserving the stability and efficiency of the learning process.
In the reward signal discretization process, the reward space $\mathcal{R} \in [\tilde{r}_{\min},\ \tilde{r}_{\max}]$ over consecutive time steps is divided into $N$ equal-width bins, and the empirical distribution within a sliding window of size $w$ is computed following the method in autocorrelation constraints, as follows.
\begin{eqnarray}\label{eq:estimate}
p_{i}^{(w)}(t)=\frac{1}{w} \sum_{k=t-w+1}^{t} \mathbb{I}\left(\tilde{r}_{k} \in \mathcal{D}_{i}\right),
\end{eqnarray}
where $\mathcal{D}_i = \left[ \tilde{r}_{min} + (i - 1) \Delta r, \tilde{r}_{min} + i \Delta r \right)$ is the $i$-th reward interval, and $\Delta i=(\tilde{r}_{max}-\tilde{r} _{min})/N$. The corresponding time-varying entropy function, based on this distribution, is expressed as follows:
\begin{eqnarray}\label{eq:function}
H_t = -\sum_{i=1}^{N} p_i^{(w)}(t) \log p_i^{(w)}(t).
\end{eqnarray}
To prevent abrupt changes in the reward distribution caused by temporally coupled perturbation attacks, entropy variation rate constraints are introduced as follows:
\begin{eqnarray}\label{eq:reward constraint2}
\mathcal{C}_{\text{ent}} = \left\| {H_t - H_{t-1}} \right\|_{\infty} \leq \epsilon_{\text{ent}},
\end{eqnarray}

where $\epsilon_{\text{ent}}$ is the maximum allowed rate of entropy change. This constraint ensures that entropy changes between consecutive time steps stay within a bounded range, promoting a smooth evolution of the reward distribution.

\subsection{Practical Implementation}

We describe the practical implementation of \texttt{TCRL}.
We optimize the strategy in the state affected by the attacker, denoted as $\tilde{\pi}$. 
The agent’s interaction with the environment under attack yields a trajectory $\tilde{\tau}=\left\{\tilde{s}_{0}, \tilde{a}_{0}, \tilde{s}_{1}, \tilde{a}_{1}, \ldots\right\}$, where each action $\tilde{a}_t \in {\Omega}(\tilde{s}_t, \tilde{\pi})$.
The worst-case safety cost, estimated under worst-case perception, is constrained by a security threshold. Additionally, we introduce a reward-based dual defense constraint into the adversarial training process.
Accordingly, we formulate the CRL objective under the temporal- coupled perturbation attacker $h$ as follows:
\begin{align}\label{eq:safe RL}
\pi_{\text{adv}} = & \max_{\pi} \ \mathbb{E} \left[ \sum_{t=0}^{T} \gamma^t {r}_t \right] \nonumber \\
\text{s.t.} & \ \underline{V}_{c}^{\tilde{\pi}} \leq \eta, \ \mathcal{C}_{\text{corr}} \leq \epsilon_{\text{corr}}, \ \mathcal{C}_{\text{ent}} \leq \epsilon_{\text{ent}} ,\forall h.
\end{align}
We incorporate the constraints into strategy optimization using Lagrangian relaxation, as follows:
\begin{align}\label{eq:non_constrain}
&(\pi _{adv},\lambda _{adv})=\max_{\pi } \min_{\lambda \geq 0} \mathbb{E} \left[ \sum_{t=0}^{T} \gamma^t {r}_t \right]-\lambda _{cost}(\underline{V}_{c}^{\tilde{\pi}}-\eta ) \nonumber\\
&- \lambda_{\text{corr}} (\mathcal{C}_{\text{corr}} - \epsilon_{\text{corr}}) - \lambda_{\text{ent}} (\mathcal{C}_{\text{ent}} - \epsilon_{\text{ent}}).
\end{align}
The Lagrangian algorithm is derived by solving the inner maximization (primal update) using a policy optimization method and performing the outer minimization (dual update) via gradient descent. Given appropriate learning rates and bounded noise, the iterates $(\pi_{\text{adv}}, \lambda_{\text{adv}})$ converge almost surely to a fixed point corresponding to a local minimum. The pseudocode of \texttt{TCRL} algorithm is shown in Appendix.

\section{Experiment}
\begin{table*}[t]
\caption{The natural performance (in the absence of attacks) of various training methods and the performance under four additional types of adversarial attacks. Each experimental result is averaged over 50 episodes and 10 random seeds, and is reported as the mean $\pm$ standard deviation. For each environment, we highlight (by shading) the agent with the lowest safety cost and the best overall performance.}

\centering
\resizebox{\textwidth}{!}{ 
\renewcommand{\arraystretch}{1.3}
\begin{tabular}{|c|c|cc|cc|cc|cc|cc|}
\hline
                               &                          & \multicolumn{2}{c|}{Natural}                                           & \multicolumn{2}{c|}{MAD}                                                & \multicolumn{2}{c|}{MC}                                                  & \multicolumn{2}{c|}{Random}                                            & \multicolumn{2}{c|}{Worst-TC}                    \\ \cline{3-12} 
\multirow{-2}{*}{\textbf{Env}} & \multirow{-2}{*}{Method} & \multicolumn{1}{c|}{Reward}        & Cost                              & \multicolumn{1}{c|}{Reward}         & Cost                              & \multicolumn{1}{c|}{Reward}         & Cost                               & \multicolumn{1}{c|}{Reward}        & Cost                              & \multicolumn{1}{c|}{Reward}        & Cost                              \\ \hline
                               & PPOL-vanilla             & \multicolumn{1}{c|}{807.66$\pm$183.5}  & 7.78$\pm$0.98                         & \multicolumn{1}{c|}{728.05$\pm$1483.51} & 20.80$\pm$76.32                       & \multicolumn{1}{c|}{903.98$\pm$595.35}  & 76.36$\pm$5.61                         & \multicolumn{1}{c|}{813.45$\pm$319.63} & 7.16$\pm$15.77                        & \multicolumn{1}{c|}{630.73$\pm$234.36} & 75.44$\pm$6.73                        \\
                               & PPOL-random              & \multicolumn{1}{c|}{764.18$\pm$25.91}  & 1.43$\pm$5.37                         & \multicolumn{1}{c|}{727.93$\pm$1750.65} & 21.36$\pm$75.31                       & \multicolumn{1}{c|}{830.64$\pm$340.04}  & 65.30$\pm$11.37                        & \multicolumn{1}{c|}{776.29$\pm$263.83} & 1.16$\pm$3.13                         & \multicolumn{1}{c|}{600.31$\pm$170.30} & 72.80$\pm$8.04                        \\
                               & ADV-PPOL(MC)             & \multicolumn{1}{c|}{559.14$\pm$258.29} & 0.82$\pm$3.01                         & \multicolumn{1}{c|}{543.56$\pm$700.39}  & 0.38$\pm$1.46                         & \multicolumn{1}{c|}{724.32$\pm$288.51}  & \cellcolor[HTML]{C0C0C0}8.58$\pm$39.47 & \multicolumn{1}{c|}{566.98$\pm$263.83} & 0.44$\pm$1.48                         & \multicolumn{1}{c|}{521.57$\pm$374.80} & 25.32$\pm$27.69                       \\ \cline{2-12} 
\multirow{-4}{*}{Ball-Circle}  & TCRL-PPOL                & \multicolumn{1}{c|}{669.34$\pm$206.81} & \cellcolor[HTML]{C0C0C0}0.00$\pm$0.00 & \multicolumn{1}{c|}{596.45$\pm$767.27}  & \cellcolor[HTML]{C0C0C0}0.00$\pm$0.00 & \multicolumn{1}{c|}{754.40$\pm$396.59}  & 15.56$\pm$8.01                         & \multicolumn{1}{c|}{550.30$\pm$264.68} & \cellcolor[HTML]{C0C0C0}0.00$\pm$0.00 & \multicolumn{1}{c|}{709.40$\pm$253.81} & \cellcolor[HTML]{C0C0C0}3.84$\pm$6.45 \\ \hline
                               & PPOL-vanilla             & \multicolumn{1}{c|}{1132.48$\pm$28.65} & 4.55$\pm$6.58                         & \multicolumn{1}{c|}{1027.35$\pm$24.91}  & 48.98$\pm$161.37                      & \multicolumn{1}{c|}{1208.97$\pm$30.67}  & 10.59$\pm$2.06                         & \multicolumn{1}{c|}{1090.53$\pm$56.17} & 5.20$\pm$1.03                         & \multicolumn{1}{c|}{1058.42$\pm$24.91} & 241.54$\pm$17.11                      \\
                               & PPOL-random              & \multicolumn{1}{c|}{1119.91$\pm$1.09}  & \cellcolor[HTML]{C0C0C0}0.00$\pm$0.00 & \multicolumn{1}{c|}{1059.68$\pm$8.38}   & 26.92$\pm$10.62                       & \multicolumn{1}{c|}{1244.07$\pm$2.62}   & 141.37$\pm$0.79                        & \multicolumn{1}{c|}{1098.36$\pm$2.33}  & 0.18$\pm$0.41                         & \multicolumn{1}{c|}{1044.45$\pm$1.76}  & 140.37$\pm$0.79                       \\
                               & ADV-PPOL(MC)             & \multicolumn{1}{c|}{1007.77$\pm$76.88} & \cellcolor[HTML]{C0C0C0}0.00$\pm$0.00 & \multicolumn{1}{c|}{968.39$\pm$67.11}   & \cellcolor[HTML]{C0C0C0}0.00$\pm$0.00 & \multicolumn{1}{c|}{1154.71$\pm$253.80} & \cellcolor[HTML]{C0C0C0}4.58$\pm$7.77  & \multicolumn{1}{c|}{972.83$\pm$44.87}  & \cellcolor[HTML]{C0C0C0}0.00$\pm$0.00 & \multicolumn{1}{c|}{961.73$\pm$50.87} & 29.32$\pm$20.01                       \\ \cline{2-12} 
\multirow{-4}{*}{Ball-Run}     & TCRL-PPOL                & \multicolumn{1}{c|}{1031.59$\pm$9.89}  & \cellcolor[HTML]{C0C0C0}0.00$\pm$0.00 & \multicolumn{1}{c|}{971.26$\pm$35.13}   & \cellcolor[HTML]{C0C0C0}0.00$\pm$0.00 & \multicolumn{1}{c|}{1154.83$\pm$23.59}  & 6.62$\pm$7.44                          & \multicolumn{1}{c|}{990.67$\pm$27.69}  & \cellcolor[HTML]{C0C0C0}0.00$\pm$0.00 & \multicolumn{1}{c|}{1154.93$\pm$33.51} & \cellcolor[HTML]{C0C0C0}1.56$\pm$1.41 \\ \hline
                               & PPOL-vanilla             & \multicolumn{1}{c|}{174.84$\pm$20.03}  & 0.18$\pm$1.59                         & \multicolumn{1}{c|}{175.13$\pm$28.82}   & 0.38$\pm$5.72                         & \multicolumn{1}{c|}{197.35$\pm$15.44}   & 25.58$\pm$35.77                        & \multicolumn{1}{c|}{174.63$\pm$48.77}  & 0.22$\pm$2.37                         & \multicolumn{1}{c|}{151.28$\pm$46.69}  & 2.72$\pm$3.41                         \\
                               & PPOL-random              & \multicolumn{1}{c|}{156.29$\pm$24.68}  & 0.48$\pm$4.37                         & \multicolumn{1}{c|}{148.34$\pm$75.83}   & 1.38$\pm$16.08                        & \multicolumn{1}{c|}{158.24$\pm$32.05}   & 40.70$\pm$77.29                        & \multicolumn{1}{c|}{151.39$\pm$38.06}  & 0.38$\pm$3.72                         & \multicolumn{1}{c|}{150.39$\pm$32.28}  & 38.42$\pm$27.04                       \\
                               & ADV-PPOL(MC)             & \multicolumn{1}{c|}{153.50$\pm$35.03}  & 3.70$\pm$8.44                         & \multicolumn{1}{c|}{109.50$\pm$43.42}   & 5.38$\pm$3.20                         & \multicolumn{1}{c|}{114.44$\pm$42.41}   & \cellcolor[HTML]{C0C0C0}2.74$\pm$3.03  & \multicolumn{1}{c|}{112.92$\pm$37.55}  & 3.28$\pm$3.72                         & \multicolumn{1}{c|}{115.59$\pm$60.51}  & 2.80$\pm$4.64                         \\ \cline{2-12} 
\multirow{-4}{*}{Car-Circle}   & TCRL-PPOL                & \multicolumn{1}{c|}{160.51$\pm$26.22}  & \cellcolor[HTML]{C0C0C0}0.00$\pm$0.00 & \multicolumn{1}{c|}{159.06$\pm$34.02}   & \cellcolor[HTML]{C0C0C0}0.00$\pm$0.00 & \multicolumn{1}{c|}{178.41$\pm$35.00}   & 7.58$\pm$8.08                          & \multicolumn{1}{c|}{179.38$\pm$30.92}  & \cellcolor[HTML]{C0C0C0}0.00$\pm$0.00 & \multicolumn{1}{c|}{177.19$\pm$41.53}  & \cellcolor[HTML]{C0C0C0}2.38$\pm$1.53 \\ \hline
                               & PPOL-vanilla             & \multicolumn{1}{c|}{813.63$\pm$34.47}  & \cellcolor[HTML]{C0C0C0}0.00$\pm$0.00 & \multicolumn{1}{c|}{769.89$\pm$13.49}   & 0.22$\pm$1.93                         & \multicolumn{1}{c|}{881.62$\pm$48.90}   & 282$\pm$18.98                          & \multicolumn{1}{c|}{808.08$\pm$20.02}  & 54.52$\pm$13.06                       & \multicolumn{1}{c|}{683.53$\pm$49.84}  & 30.78$\pm$7.95                         \\
                               & PPOL-random              & \multicolumn{1}{c|}{817.78$\pm$0.98}   & \cellcolor[HTML]{C0C0C0}0.00$\pm$0.00 & \multicolumn{1}{c|}{792.69$\pm$6.49}    & 1.20$\pm$2.32                         & \multicolumn{1}{c|}{870.49$\pm$0.51}    & 58.58$\pm$79.6                         & \multicolumn{1}{c|}{813.29$\pm$2.18}   & \cellcolor[HTML]{C0C0C0}0.00$\pm$0.00 & \multicolumn{1}{c|}{670.37$\pm$0.50}   & 34.52$\pm$16.81                       \\
                               & ADV-PPOL(MC)             & \multicolumn{1}{c|}{761.68$\pm$1.10}   & \cellcolor[HTML]{C0C0C0}0.00$\pm$0.00 & \multicolumn{1}{c|}{703.02$\pm$36.92}   & 0.06$\pm$0.74                         & \multicolumn{1}{c|}{855.55$\pm$0.73}    & \cellcolor[HTML]{C0C0C0}0.44$\pm$0.25  & \multicolumn{1}{c|}{753.66$\pm$5.16}   & \cellcolor[HTML]{C0C0C0}0.00$\pm$0.00 & \multicolumn{1}{c|}{597.73$\pm$0.39}   & 0.36$\pm$0.23                         \\ \cline{2-12} 
\multirow{-4}{*}{Car-Run}      & TCRL-PPOL                & \multicolumn{1}{c|}{759.92$\pm$1.11}   & \cellcolor[HTML]{C0C0C0}0.00$\pm$0.00 & \multicolumn{1}{c|}{683.60$\pm$24.40}   & \cellcolor[HTML]{C0C0C0}0.00$\pm$0.00 & \multicolumn{1}{c|}{861.88$\pm$0.66}    & 0.72$\pm$0.24                          & \multicolumn{1}{c|}{751.02$\pm$3.41}   & \cellcolor[HTML]{C0C0C0}0.00$\pm$0.00 & \multicolumn{1}{c|}{899.87$\pm$0.49}   & \cellcolor[HTML]{C0C0C0}0.18$\pm$0.26 \\ \hline

\end{tabular}}
\label{tab:results}\end{table*}
\subsection{Adversarial Attacker}

To assess the robustness of \texttt{TCRL} against temporal-coupled perturbations, we design a well-trained worst-case temporal-coupled attacker (Worst-TC).
In addition, we include three well-trained adversarial baselines without temporal-coupled perturbations to assess whether \texttt{TCRL} is also robust against temporal-independent perturbations, all constrained by the same $\ell_{\infty}$-norm perturbation set $\mathcal{B}^{\epsilon}_{\infty}$.

\begin{itemize}[leftmargin=*]

\item \noindent \textbf{Worst-case Temporal-coupled (Worst-TC) attacker}. 
The attacker is constructed by integrating the worst-case cost estimation network with a reward minimization estimator, enabling adversarial behavior that simultaneously maximizes cost and minimizes reward. 
Let the agent's strategy be denoted by $\pi$. The attacker perturbs the observed state $\tilde{s}_t \in \mathcal{B}_{\epsilon}(s_t)$ using the temporal-coupled perturbation attack described in Methodology, causing the agent to take an action $\tilde{a}_t \in {\Omega}(\tilde{s}_t, \pi)$. The goal of the worst-case temporally coupled attacker is to minimize the agent’s cumulative reward while maximizing its cumulative cost through such perturbations. 
For a given strategy $\pi$ and its critic $Q ^ {\pi} _r \left (s_ {t}, a \right) $ and ${Q} ^ {\pi} _c \left (s_ {t}, a \right) $, we formulate the attackers as follows:
\begin{eqnarray}\label{eq:attacker}
\mathcal{H}(s)=\arg \min _{a \in {\Omega}({s}_t, \pi)}\left[Q^{\pi}_r\left(s_{t}, a\right)-\lambda {Q}^{\pi}_c\left(s_{t}, a\right)\right]
\end{eqnarray}
where $Q^{\pi}_r(s,a)$ is the standard action value function of the agent, $C^{\pi}\left(s_{t}, a\right)={Q}^{\pi}_c\left(s_{t}, a\right)$ is the cost constraint function, and $\lambda < 1$ is a weighting parameter that controls the importance of the cost.

\item \noindent \textbf{Random attacker baseline}. This method emulates naive attack behavior by applying random perturbations to the agent's observation or action space, without any specific objective or strategy.

\item \noindent \textbf{Max cost (MC) attacker baseline}. 
This method is designed to increase the safety cost of the trained agent, thereby increasing the likelihood of violating safety constraints and inducing unsafe behavior. It generates adversarial observations by maximizing the expected cost over the perturbed state space, expressed as: $\nu _{MC}(s)=arg\max_{\tilde{s}\in \mathcal{B}_{p}^{\epsilon } (s) } \mathbb{E}_{\tilde{a}\sim \pi (a|\tilde{s})}[Q_ {c}^{\pi }(s,\tilde{a}) ]$.

\item \noindent \textbf{Maximum Action Difference (MAD) attacker baseline}. The MAD attacker has been shown to be effective in reducing the reward value of trained reinforcement learning agents. It obtains optimal adversarial observations by maximizing the KL scatter between the corrupted strategies. The attack strategy is expressed as: $\nu _{MAD}(s)=arg\max_{\tilde{s}\in \mathcal{B}_{p}^{\epsilon } (s) } D_{KL}[\pi (a|\tilde{s})\left |  \right |\pi (a|s) ]$.

\end{itemize}

\subsection{Experimental Settings}

\noindent\textbf{Implementation details.} For the evaluated tasks, we selected robotic motion control as the test domain. The simulated environments are adopted from a previous benchmark~\cite{gronauer2022bullet}.

\begin{table*}[t]
\centering
\caption{Ablation study on Ball-Circle, Ball-Run, Car-Circle and Car-Run}
\vspace{-1.5mm}
\resizebox{\textwidth}{!}{ 
\renewcommand{\arraystretch}{1.1}
\begin{tabular}{|c|c|cc|c|c|cc|}
\hline
                              &                          & \multicolumn{2}{c|}{Worst-TC}                                                            &                            &                          & \multicolumn{2}{c|}{Worst-TC}                                                            \\ \cline{3-4} \cline{7-8} 
\multirow{-2}{*}{Env}         & \multirow{-2}{*}{Method} & \multicolumn{1}{c|}{Reward}                 & Cost                                       & \multirow{-2}{*}{Env}      & \multirow{-2}{*}{Method} & \multicolumn{1}{c|}{Reward}                 & Cost                                       \\ \hline
                              & TCRL-Worstcost           & \multicolumn{1}{c|}{698.76$\pm$176.78}          & \cellcolor[HTML]{FFFFFF}17.54$\pm$10.21        &                            & TCRL-Worstcost           & \multicolumn{1}{c|}{1168.51$\pm$2.78}           & \cellcolor[HTML]{FFFFFF}27.60$\pm$53.44        \\
                              & TCRL-reward              & \multicolumn{1}{c|}{588.28$\pm$219.53}          & \cellcolor[HTML]{FFFFFF}0.00$\pm$0.00          &                            & TCRL-reward              & \multicolumn{1}{c|}{1065.55$\pm$17.43}          & \cellcolor[HTML]{FFFFFF}2.73$\pm$5.27          \\
\multirow{-3}{*}{Ball-Circle} & \textbf{TCRL}            & \multicolumn{1}{c|}{\textbf{709.40$\pm$253.81}} & \cellcolor[HTML]{FFFFFF}\textbf{3.84$\pm$6.45} & \multirow{-3}{*}{Ball-Run} & \textbf{TCRL}            & \multicolumn{1}{c|}{\textbf{1154.93$\pm$33.51}} & \cellcolor[HTML]{FFFFFF}\textbf{1.56$\pm$1.41} \\ \hline
                              & TCRL-Worstcost           & \multicolumn{1}{c|}{163.66$\pm$11.05}           & \cellcolor[HTML]{FFFFFF}4.75$\pm$15.71         &                            & TCRL-Worstcost           & \multicolumn{1}{c|}{862.03$\pm$27.46}           & \cellcolor[HTML]{FFFFFF}20.52$\pm$19.89        \\
                              & TCRL-reward              & \multicolumn{1}{c|}{125.73$\pm$53.03}           & \cellcolor[HTML]{FFFFFF}2.46$\pm$4.25          &                            & TCRL-reward              & \multicolumn{1}{c|}{744.91$\pm$1.21}            & \cellcolor[HTML]{FFFFFF}0.00$\pm$0.00          \\
\multirow{-3}{*}{Car-Circle}  & \textbf{TCRL}            & \multicolumn{1}{c|}{\textbf{177.19$\pm$41.53}}  & \cellcolor[HTML]{FFFFFF}\textbf{2.38$\pm$1.53} & \multirow{-3}{*}{Car-Run}  & \textbf{TCRL}            & \multicolumn{1}{c|}{\textbf{899.87$\pm$0.49}}   & \cellcolor[HTML]{FFFFFF}\textbf{0.18$\pm$0.26} \\ \hline
\end{tabular}}
\label{tab:Ablation}
\vspace{-1.5mm}
\end{table*}

\begin{itemize}[leftmargin=*]
\item \noindent \textbf{Circle task}. 
This task targets two objectives: (i) accurate trajectory tracking and (ii) safe obstacle avoidance, as the agent navigates around a central circular hazard zone. The agent must perceive environmental information in real time and employ advanced perception and decision-making algorithms to strictly avoid entering the hazardous area and effectively prevent collisions with obstacles, thereby ensuring system safety and task reliability.

\item \noindent \textbf{Run task}.
This task evaluates the agent’s ability to operate efficiently within a movement space constrained by dual safety boundaries. The agent is rewarded based on its speed, path efficiency, and related performance metrics, provided it maintains high-speed motion within the feasible region through effective planning and control strategies. Negative rewards (i.e., cost penalties) are incurred if the agent crosses a safety boundary, enters a hazardous area, or exceeds its designated speed limit, with penalties scaled by the severity and duration of the violation.
\end{itemize}

We adopt two types of robots (Ball and Car) and assign them to two tasks (Circle and Run), yielding four Robot-Task combinations: (i) \underline{\textbf{\textit{Ball-Circle}}}, (ii) \underline{\textbf{\textit{Ball-Run}}}, (ii) \underline{\textbf{\textit{Car-Circle}}}, and (iv) \underline{\textbf{\textit{Car-Run}}}.

To validate the effectiveness of \texttt{TCRL}, we integrate it with the PID-PPO-Lagrange(PPOL) framework~\cite{stooke2020responsive} to form \texttt{TCRL-PPOL}. More details about PPOL are shown in Appendix.




\noindent\textbf{Baselines.}
We present three baseline robust training methods, which are alternately trained alongside adversarial attackers to enhance robustness.
\begin{itemize}[leftmargin=*]
\item \noindent \textbf{PPOL-vanilla} adopts the PID-PPO-Lagrange method~\cite{stooke2020responsive} and does not incorporate any form of robustness training.

\item \noindent \textbf{PPOL-random} is trained with attackers perturbed by random noise.

\item \noindent \textbf{ADV-PPOL(MC)} is a state-based adversarial training method that employs a carefully designed Max-Cost (MC) attacker.
\end{itemize}

All methods were implemented using the same PPOL algorithm and hyperparameter settings across test environments to ensure a fair comparison.

\vspace{-3mm}
\subsection{Performance of CRL with Robust Training}

We evaluated three baseline training methods and our proposed \texttt{TCRL-PPOL} under five conditions: (i) no perturbation, (ii) MAD attacker, (iii) Max-Cost (MC) attacker, (iv) random attacker, and (v) worst-case temporally coupled (Worst-TC) attacker.

Notably, a superior method should achieve high rewards and low safety costs under diverse attacks, allowing agents to complete tasks with higher quality while maintaining safety. To this end, we jointly consider both reward and safety cost in our evaluation.

Table~\ref{tab:results} reports that \texttt{TCRL-PPOL} outperforms all three baselines in defending against state perturbations across the four tasks (Ball-Circle, Ball-Run, Car-Circle and Car-Run). Moreover, we make these observations:

First, we observe that although most baselines achieve near-zero or relatively low safety costs under natural conditions, their safety performance degrades significantly when subjected to the Worst-TC attacker, which employs temporal-coupled perturbations.
\texttt{TCRL} consistently demonstrates superior safety performance compared to all baselines, achieving the lowest safety cost even under the temporal-coupled perturbations introduced by the Worst-TC attacker. Notably, \texttt{TCRL} achieves higher rewards under Worst-TC attacks than in natural scenarios due to its focus on cost constraints. 
In natural environments, \texttt{TCRL} employs a defense-oriented strategy that prioritizes safer, lower-cost actions, even if it means sacrificing some rewards.
This strategy becomes advantageous under Worst-TC attack scenarios, allowing the policy to perform more robustly against adversarial perturbations, while maintaining higher rewards than baselines.

Compared to the three baseline training methods, \texttt{TCRL} achieves significantly lower safety costs against temporal-coupled perturbations across all tasks: a reduction ranging from 559.38\% to 1864.58\% in Ball-Circle, 1779.49\% to 15,383.33\% in Ball-Run, 14.29\% to 1514.29\% in Car-Circle, and 100.00\% to 19,077.78\% in Car-Run.
This is because the specialized safety cost constraint function in \texttt{TCRL}, which is designed for worst-case cost estimation, enables more accurate constraint enforcement under temporally coupled perturbations, thereby improving the safety of CRL in adversarial settings.

Second, in the comparison with the three baselines, TCRL also demonstrated good safety performance under temporal-independent perturbation attacks and could maintain a low safety cost under MAD, MC and Random attackers. Further, ADV-PPOL (MC) consistently achieves the lowest safety cost when attacked by the MC attacker, as it is specifically trained under state perturbations introduced by this attacker.
However, due to the lack of defense strategies, its safety cost remains relatively high when attacked by temporal-coupled perturbations.

Third, under attacks from the Worst-TC attacker$-$a temporally coupled perturbation adversary$-$the reward values of the three baselines decrease significantly compared to their performance under other attacks and natural conditions.
However, \texttt{TCRL} achieves a higher reward under the Worst-TC attack than it does under natural conditions, and its reward also surpasses those of all three baselines.
Compared to the three baselines, \texttt{TCRL} improves reward values by 11.09\% to 26.47\% in Ball-Circle, 8.36\% to 16.73\% in Ball-Run, 14.62\% to 34.76\% in Car-Circle, and 24.04\% to 33.58\% in Car-Run.
This is because the dual-constraint defense mechanism in \texttt{TCRL}, based on reward signals, can effectively mitigate reward-targeted attacks from temporal-coupled perturbation attackers and enhance the robustness of CRL against such perturbations.

\vspace{-3mm}
\subsection{Ablation Study} 
We investigate the impact of the proposed cost constraint functions and reward-based temporal defense constraints on enhancing the robustness of \texttt{TCRL} against temporal-coupled perturbation attacks. 
To evaluate the contribution of each component, we define two ablated variants: (i) \texttt{TCRL-Worstcost}, which removes the cost constraint functions, and (ii) \texttt{TCRL-Reward}, which omits the reward-based temporal defense constraints.
Table~\ref{tab:Ablation} reports the experimental results. Compared to the full \texttt{TCRL}, we make these observations:



First, removing the cost constraint function in \texttt{TCRL- Worstcost} results in an increase in safety cost by 13.70 in the Ball-Circle task, 26.04 in Ball-Run, 2.37 in Car-Circle, and 20.34 in Car-Run. This increase indicates that, without explicit penalization of unsafe behavior, the agent becomes more susceptible to temporal-coupled perturbations.


Second, removing the reward-based temporal defense constraints in \texttt{TCRL-Reward} leads to a significant reduction in reward: 121.12 in Ball-Circle, 89.38 in Ball-Run, 51.46 in Car-Circle, and 154.96 in Car-Run. This degradation demonstrates that, without temporal reward guidance, the agent struggles to maintain effective behavior under temporal-coupled perturbations.

\section{Conclusion}
We propose \texttt{TCRL}, a robust CRL method enhancing agent robustness under temporally coupled state perturbations via a novel worst-case training framework. It includes a safety cost constraint function for per-state worst-case cost estimation and a dual-constraint reward-based defense to disrupt adversaries' temporal coupling while preserving reward unpredictability. Experiments show \texttt{TCRL} effectively defends against worst-case temporal-coupled perturbations and maintains safety under temporal-independent attacks.






\bibliographystyle{ACM-Reference-Format} 
\bibliography{sample}


\clearpage
\appendix
\section{Appendix}
\subsection{Examples in  Temporal-coupled State Perturbations}

We provide two illustrative examples to highlight the importance of enhancing robust CRL methods against temporal-coupled state perturbations.

 \begin{example}
In the physical deployment of autonomous driving systems, directional (high-intensity light) interference on vehicle-mounted cameras at adjacent time steps constitutes a time-coupled adversarial perturbation attack. Attackers leverage laser devices on road sides or rear vehicles to continuously emit specific-wavelength beams toward the target vehicle's camera, with the beam intensity gradually increasing over time to simulate scenarios such as oblique morning sunlight. This temporal-coupled interference causes linear saturation of image sensors: initial deviations in lane line recognition; as light intensity rises, the system misidentifies glare as lane markers and continuously steers in the opposite direction, eventually leading the vehicle to enter oncoming lanes and maximizing the safety cost of collision risks. Tests conducted at the University of Michigan's MCity confirm that such progressive optical attacks can induce long-term performance degradation and safety violations without triggering immediate alarms.
 \end{example}

\begin{example}
Temporal-coupled attackers in warehouse robotic navigation systems actively compute real-time gradients to execute attacks. After hijacking indoor positioning signals, they inject directionally consistent perturbations with monotonically escalating magnitudes(e.g., progressive X-axis offsets). Malicious vectors are dynamically optimized to simultaneously minimize path-tracking rewards and maximize obstacle proximity(safety cost). This dual-objective attack reduces operational efficiency and, upon reaching cumulative offset thresholds, drives robots into hazardous material zones. Emergency braking induces cargo destruction and multi-robot cascade failures, maximizing safety costs asymptotically.
\end{example}
Such temporally coupled perturbation attacks affect decision-making at individual time steps and, through cumulative effects, influence the agent’s behavior over subsequent time steps.
This significantly compromises the agent’s task performance and safety. Therefore, enhancing robustness against temporal-coupled state perturbations is critically important.

\subsection{Theoretical Analysis}

Analogous to the worst-case action cost in Formula~(\ref{eq:Worst_c}), the worst-case cost value is formulated as follows.

\begin{definition}[Worst-case Cost Value]\label{Worst-case Cost Value}For a given policy $\pi$, its worst-case cost value is:
\begin{eqnarray}\label{eq:costvalue}
\underline{V}_c^{\pi}(s) := \mathbb{E}_{\pi, P} \left[ \sum_{t = 0}^{\infty} \gamma^t c(s_t, a^*_t) |s_0=s \right],
\end{eqnarray}
where $a^*_t$ is the worst-case action taken by state $s_t$ under a temporal-coupled perturbation attack.
\end{definition}

\begin{proof}[Proof of Theorem~\ref{thm:cost}] 
We show that $\underline{ \Gamma}^{\pi}$ is a contraction.
For two $Q$-functions $\mathcal{Q}_1$, $\mathcal{Q}_2:\mathcal{S}\times\mathcal{A}\rightarrow\mathbb{R}$, we can compute:
\begin{align}
&\|\mathcal{T}^{\pi}Q_1 - \mathcal{T}^{\pi}Q_2\|_{\infty} \\
=& \max_{s,a} \left| \sum_{s' \in \mathcal{S}} P(s' \mid s,a) \left[ R(s,a) + \gamma \min_{a' \in {\Omega}(s',\pi)} Q_1(s',a') \right.\right.\\ \nonumber 
&\ \left.\left. - R(s,a) + \gamma \min_{a' \in {\Omega}(s',\pi)} Q_2(s',a') \right] \right| \\
=& \gamma \max_{s,a} \left| \sum_{s' \in \mathcal{S}} P(s' \mid s,a) \left[ \min_{a' \in {\Omega}(s',\pi)} Q_1(s',a') \right.\right.\\ \nonumber
&\ \left.\left.- \min_{a' \in {\Omega}(s',\pi)} Q_2(s',a') \right] \right| \\
\leq& \gamma \max_{s,a} \sum_{s' \in \mathcal{S}} P(s' \mid s,a) \left| \min_{a' \in {\Omega}(s',\pi)} Q_1(s',a') \right.\\ \nonumber
&\ \left.- \min_{a' \in {\Omega}(s',\pi)} Q_2(s',a') \right| \\
\leq& \gamma \max_{s,a} \sum_{s' \in \mathcal{S}} P(s' \mid s,a) \max_{a' \in {\Omega}(s',\pi)} |Q_1(s',a') - Q_2(s',a')| \\
=& \gamma \max_{s,a} \sum_{s' \in \mathcal{S}} P(s' \mid s,a) \|Q_1 - Q_2\|_{\infty} \\
=& \gamma \|Q_1 - Q_2\|_{\infty}.
\end{align}

The following inequality is derived from the fact that
\begin{eqnarray}\label{eq:fact}
\left| \min_{m} f(m) - \min_{n} g(n) \right| \leq \max_{k} |f(k) - g(k)|.
\end{eqnarray}

The operator $\underline{ \Gamma}^{\pi}$ satisfies the following inequality:
\begin{eqnarray}\label{eq:operator}
\|\mathcal{T}^\pi Q_1 - \mathcal{T}^\pi Q_2\|_\infty \leq \gamma \|Q_1 - Q_2\|_\infty.
\end{eqnarray}
Thus, it constitutes a contraction with respect to the supremum norm.

As shown in Formula~(\ref{eq:Worst_c}), the worst-case action cost value is given by: $\underline{Q}^{\pi}_c\left(s,a\right):= \mathbb{E}_{\pi,P} \left[ \sum_{t=0}^{\infty} \gamma^{t} c(s_{t},a^{\ast}_{t})  \right]$, where $a^*_t$ is the worst-case action taken by state $s_t$ under a temporal-coupled perturbation attack. 
It follows that the temporal-coupled attacker induces the agent to select the worst possible action among all the achievable actions in $\Omega$.
As a result, we obtain $\underline {Q}_c ^ \pi(s, a) = \underline{ \Gamma}^{\pi} \underline {Q}_c ^ \pi (s, a) $. Therefore, $\underline{Q}_c^\pi(s, a)$is the fixed point of the Bellman operator.
\end{proof}

\subsection{PPO-Lagrangian Algorithm}
We use the gradient of the state-action value function $Q(s,a)$ to provide the direction to update states in $N$ steps:
\begin{eqnarray}\label{eq:step}
&s^{n+1} = \text{Proj}[s^n - \rho \nabla_{s^n} Q(s^0, \pi(s^n))], \\ \nonumber
& \quad n = 0, \ldots, n-1.
\end{eqnarray}
where $\text{Proj}[\cdot]$ is a projection onto $B_p^{\overline{\epsilon}_{t}}(s^0)$, $\rho$ is the learning rate, and $s^0$ is the state being attacked. 

Given that the $Q$-function and policy are parameterized by neural networks, the gradient can be backpropagated from $Q$ to $s^n$ through $\pi(\tilde{a}|s^n)$. This process can be efficiently solved using optimizers such as adaptive moment estimation(ADAM).

The objective of PPO (clipped) has the form ~\cite{schulman2017proximalpolicyoptimizationalgorithms}:
\begin{align}\label{eq:PPOL}
\ell_{ppo}=& \min 
\left[\frac{\pi_{\theta}(a|s)}{\pi_{\theta n}(a|s)} A^{\pi_{\theta n}}(s, a), \operatorname{clip}\left(\frac{\pi_{\theta}(a|s)}{\pi_{\theta n}(a|s)}, 1-\epsilon, \right.\right.\\ \nonumber 
&\ \left.\left. 1+\epsilon\right) A^{\pi_{\theta n}}(s, a)\right]
\end{align}

In \texttt{TRCL}, we adopt the PID Lagrangian method~\cite{stooke2020responsive}, which reduces the oscillation and overshoot encountered in traditional Lagrangian methods. 
The loss function employed in the PPO-Lagrangian method is formulate as follows:
\begin{eqnarray}\label{eq:PPOL}
\ell_{\text{ppol}} = \frac{1}{1 + \lambda} \left( \ell_{\text{ppo}} + V_r^\pi - \lambda V_c^\pi \right)
\end{eqnarray}
The Lagrangian multiplier $\lambda$ is calculated using feedback control on $ V_c^{\pi} $. During this process, the control parameters $ K_P $, $ K_I $, and $ K_D $ need to be fine - tuned.

\subsection{Pseudocode of \texttt{TRCL} Algorithm}

\begin{algorithm}[H]
\caption{\texttt{TCRL} Algorithm}
\label{pseudocode1}

\raggedright
\textbf{Input:} Number of interations $T$, CRL algorithm loss  function $\mathcal{L}_{adv}$, safety cost threshold $\eta$, reward the threshold of autocorrelation coefficient $\epsilon_{corr}$, reward maximizes the threshold of entropy change rate $\epsilon_{ent}$.\\

\textbf{Output:} Temporal-coupled robust CRL policy $\pi_{adv}$. 

\begin{algorithmic}[1]

\STATE Initialize policy\ network\ $\pi_{\theta}(a|s)$\ with\ parameters\ $\theta$.
\STATE Initialize\ worst-case\ temporal-coupled\ critic\ network\ $\underline{Q}^{\pi}_{\varphi}(s,a)$\ with\ parameters\,\ $\varphi.$
\FOR{$k=0,1,\dots ,T$}
    \STATE  Collect\ the\ trajectory\ $\mathcal{M}={\tau_{k}}$\ via\ running\ $\pi_{\theta}$\ in\ CRL\ task.
    \STATE  Find max/min bounds of \ $\pi_{\theta}$: ${\pi}_{max}$\  and\  ${\pi}_{min}$.

    \STATE  Compute action set $\Omega(s, \pi)=[{\pi}_{max},{\pi}_{min}]$.
    \STATE  Select the worst action $\tilde{a}_{t+1}$ for next states.

    \STATE  Compute worst-case cost $\underline{c}_t$.

    \STATE Update the parameter of worst case temporal coupled critic network via minimizing the cost error($\mathcal{L}_{cost}$).

    \STATE  Compute the reward autocorrelation coefficient $C_{corr}.$
    
    \STATE Compute the reward the entropy change rate $C_{ent}.$

    \STATE Update the network $\pi_{\theta}$ by minimizing $\mathcal{L}_{adv}$.

\ENDFOR
\STATE \textbf{return} $\pi_{adv}$
\end{algorithmic}
\end{algorithm}

\subsection{The Details of The Process of Computing ${{\Omega}}_{\text{adv}}(s, \pi)$}

These techniques enable the derivation of layer-wise bounds for the neural network's behavior. Specifically, we calculate ${\pi}_{max}$ and ${\pi}_{min}$ to satisfy the condition ${\pi}_{max}(s) \geq \pi(\hat{s}) \geq {\pi}_{min}(s)$, $\forall\hat{s}\in\mathcal{B}_{\epsilon}(s)$. By employing this relaxation approach, we can define a superset of ${\Omega}_{\text{adv}}$, denoted as $\hat{{\Omega}}_{\text{adv}}$. Then, by replacing ${\Omega}_{\text{adv}}$ by $\hat{{\Omega}}_{\text{adv}}$, the fix point of Equation~(\ref{eq:Bellman}) becomes a lower bound of the worst-case cost value. For continuous action spaces, $\hat{{\Omega}}_{\text{adv}}(s, \pi)$ contains actions that are constrained by ${\pi}_{max}(s)$ and ${\pi}_{min}(s)$. For discrete action space, we can first calculate the maximum and minimum probabilities of taking each action and derive the possible set of actions that can be selected.

\subsection{The Details of Hyperparameter}

We use a Gaussian policy across all experiments, where the mean is generated by the neural network and the variance is learned independently.
In the Car-Run task, both the policy and $Q$ networks are composed of two hidden layers with sizes (128, 128). For the other three tasks (Ball-Circle, Ball-Run and Car-Circle), both the policy and $Q$ networks consist of two hidden layers with sizes of (256, 256). We use the ReLU activation function for all tasks. The discount factor is set to $\gamma = 0.99$. The advantage function is estimated using $GAE$-$\lambda$ with $\lambda^{\text{GAE}} = 0.95$. The KL divergence step size is set to $\tau_{\text{KL}} = 0.01$, and the clipping coefficient is set to $\epsilon_c = 0.02$.
In the robotic control tasks, the PID parameters are set to $K_p = 0.1$, $K_i = 0.005$, and $K_d = 0.002$. The learning rate for all attackers is set to 0.05. The step size of the Adam optimizer is set to $3 \times 10^{-4}$, and the threshold is set to $\xi = 0.1$.
The set of hyperparameters in the experiments is provided in Table~\ref{parameter}.

\begin{table}[h]
\caption{Hyperparameter of Experiment.}
\centering
\begin{threeparttable}
\begin{adjustbox}{width=\columnwidth}
\renewcommand{\arraystretch}{1.4}
\begin{tabular}{c|c|c|c|c}
\hline
Parameter                 & Ball - Circle & Ball - Run & Car - Circle & Car - Run \\ \hline
training epoch            & 100         & 100      & 150        & 150     \\ \hline
batch size                & 40000       & 40000    & 40000      & 40000   \\ \hline
minibatch                 & 300         & 300      & 300        & 300     \\ \hline
rollout length            & 200         & 200      & 300        & 200     \\ \hline
cost limit $\eta$         & 5           & 5        & 5          & 5       \\ \hline
$\epsilon_{corr}$         & 5           & 5        & 5          & 5       \\ \hline
$\epsilon_{ent}$          & 2           & 2        & 2          & 2       \\ \hline
actor optimization step M & 80          & 80       & 80         & 80      \\ \hline
actor learning rate       & 0.0002      & 0.0002   & 0.0003     & 0.0003  \\ \hline
cirtic learning rate      & 0.001       & 0.001    & 0.001      & 0.001   \\ \hline
\end{tabular}
\end{adjustbox}
\end{threeparttable}
\label{parameter}
\end{table}

\subsection{The Details of Experiment Settings}
The relevant results were obtained on an Intel Core i7-12700 CPU with 26GB of RAM. Our operation system is Ubuntu 22.04.5 LTS(x64).

\end{document}